\documentclass{article} 

\usepackage{tikz}
\usetikzlibrary{arrows}

\newcommand{\YY}{\vec{Y}}
\newcommand{\yy}{\vec{y}}
\newcommand{\XX}{\vec{X}}
\newcommand{\xx}{\vec{x}}
\newcommand{\ZZ}{\vec{Z}}
\newcommand{\zz}{\vec{z}}

\newcommand{\uu}{\vec{u}}
\usepackage{amsthm}
\usepackage{amssymb}
\newcommand{\<}{\langle}
\renewcommand{\>}{\rangle}
\newcommand{\rimp}{\Rightarrow}
\newcommand{\dimp}{\Leftrightarrow}

\usepackage{enumitem}

\renewcommand{\phi}{\varphi}

\renewcommand{\S}{{\cal S}}
\newcommand{\F}{{\cal F}}
\newcommand{\U}{{\cal U}}
\newcommand{\R}{{\cal R}}
\newcommand{\V}{{\cal V}}

\newcommand{\C}{{\cal C}}
\newcommand{\union}{\cup}
\newcommand{\sat}{\models}
\newcommand{\disc}{{\it disc}}
\newcommand{\LS}{{\cal L}(\S)}
\newcommand{\LSd}{{\cal L}^d(\S)}
\newcommand{\MSc}{{\cal M}_c^{\S}}

\newcommand{\MS}{{\cal M}^{\S}}

\newcommand{\AXp}{\mathit{AX}^+}
\newcommand{\AX}{\mathit{AX}}

\newcommand{\AXpd}{\mathit{AX}^{+,d}}

\newcommand{\commentout}[1]{}
\newcommand{\true}{{\it true}}

\newcommand{\dfn}{\begin{definition}}
\newcommand{\bbox}{\vrule height7pt width4pt depth1pt}
\newcommand{\edfn}{\bbox\end{definition}}
\newcommand{\thm}{\begin{theorem}}
\newcommand{\ethm}{\end{theorem}}
\newtheorem{proposition}{Proposition}
\newtheorem{theorem}{Theorem}
\newcommand{\pro}{\begin{proposition}}
\newcommand{\epro}{\end{proposition}}
\newtheorem{example}{Example}
\newcommand{\cor}{\begin{corollary}}
\newcommand{\ecor}{\end{corollary}}

\newcommand{\xam}{\begin{example}}
\newcommand{\exam}{\end{example}}
\newcommand{\LDL}{{\mathit LDL}}
\newcommand{\VLDL}{{\mathit VLDL}}
\newcommand{\vldl}{{\mathit vldl}}
\newcommand{\TRI}{{\mathit TRI}}

\newcommand{\HDL}{{\mathit HDL}}
\newcommand{\TOT}{{\mathit TOT}}
\newcommand{\ldl}{{\mathit ldl}}
\newcommand{\hdl}{{\mathit hdl}}
\newcommand{\tot}{{\mathit tot}}
\newcommand{\fullv}[1]{#1}
\newcommand{\shortv}{\commentout}
\newcommand{\yyp}{\vec{y}^*}

\title{Causal Models with Constraints}
\usepackage{times}
\author{Sander Beckers \\ 
University of T\"ubingen\\
srekcebrednas@gmail.com
\and
Joseph Y. Halpern\\
Cornell University\\
halpern@cs.cornell.edu
\and
Christopher Hitchcock\\
California Institute of Technology\\
cricky@caltech.edu
}
\date{}

\begin{document}
\maketitle

\begin{abstract}
Causal models have proven extremely useful in offering formal
representations of causal relationships between a set of variables. Yet in
many situations, there are non-causal  
relationships among variables.  For example, we may want variables $\LDL$, $\HDL$, 
and $\TOT$ that represent the level of low-density lipoprotein cholesterol, the
level of lipoprotein
high-density lipoprotein cholesterol, and total cholesterol level,
with the relation  
$\LDL+\HDL=\TOT$.  This cannot be done in standard causal models, because we 
can intervene simultaneously on all three variables.
 The goal of this paper is to extend
standard causal models to allow for constraints on settings of variables.
Although the extension is relatively straightforward, to make it useful we have to define
a new intervention operation that \emph{disconnects} a variable from a causal equation.  
We give examples showing the usefulness of this extension, and 
provide a sound and complete axiomatization for causal models with
constraints.
\end{abstract}

\section{Introduction}

\noindent Causal models have proven extremely useful in offering formal
representations of causal relationships between a set of variables. Yet in
many situations we want to study both causal and non-causal relationships 
between a single set of variables; this cannot be done in a standard
causal model.
For example, a standard causal model cannot talk simultaneously 
about the level of high-density lipoprotein cholesterol ($\HDL$), the
level of low-density lipoprotein cholesterol ($\LDL$), and the level of
total cholesterol ($\TOT$), although this seems quite 
natural. One can imagine a situation where we only have data
regarding 
the level of total cholesterol, even though our causal
model may say that certain health conditions depend on the amount of
$\LDL$.  The problem is that standard causal models
allow simultaneous interventions to all variables in the model.  But
we cannot intervene to simultaneously set 
$\LDL$ to
$120$ mg/dL,
$\HDL$ to 
$70$, 
and $\TOT$ 
to 
$180$,
for that is logically inconsistent!
In this example, the variables have a part-whole relationship, rather
than a causal relationship. 
Other kinds of non-causal constraints giving rise to similar problems include: 
\begin{itemize}
  \item Unit transformations; for example, having variables that
    describe weight in pounds and 
  weight in kilograms. 
\item Mathematical relationships; for example, having variables for both
Cartesian
  co-ordinates and polar co-ordinates. 
\item Microscopic/macroscopic relationships; for example, having
    variables for chemical 
  compositions combined with variables indicating whether a liquid
  sample is water, hydrogen peroxide, or sulphuric acid; or variables
  representing the distribution of molecular velocities in a sample of
  gas, together with variables representing temperature and pressure.  
\end{itemize}
Representing 
any of these 
using standard
causal models would require having separate causal models
for each 
separate
 description, thereby ignoring the important
non-causal relationships between the 
variables in the distinct models.

Allowing models with non-causal constraints increases the
expressive power of causal models in important ways.
\fullv{For one thing, we can represent {\em ambiguous} interventions.}
For example, if we change only $\TOT$,
rather than changing the levels of $\LDL$ and $\HDL$
separately, then such a change  is ambiguous, because it
can be realized in a number of different ways, corresponding to
different (and perhaps unknown) 
interventions on $\LDL$ and $\HDL$.
(This terminology, as well as the cholesterol example, are taken from
\cite{SS04}.)
Having constraints also gives us a way of effectively disallowing certain
interventions, by stipulating that certain settings of the variables are
disallowed, such as setting $\TOT$ below the sum of
$\LDL$ and $\HDL$.

Moreover, causal models with constraints have an important practical
application. In many cases, different institutions or researchers
study the same causal domain using non-causally related sets of
variables. These relationships can be as trivial as the unit
transformations mentioned above, but can also be far more complicated,
such as the relationship between particular settings and outputs of
fMRI machines produced by different companies, the translation of
terminology used in the financial reporting of different countries, or
more generally, the relationship between datasets that encode
observations of the same kind using different conventions.
We cannot combine the causal models used by such groups into one
(standard) causal model, because of the relationships between the
variables used in different models.
On the other hand, causal models with constraints
allow for the integration of the causal knowledge of the individual
models into one combined model.

The goal of this paper is to show 
how
 all of this (and
more) can be accomplished by extending causal models with constraints
 on settings of variables.  
Although the extension is relatively straightforward, to make it useful we 
have to define
a new operation. Specifically, we need to
be able to \emph{disconnect} a variable from a causal equation.  
We provide examples that illustrate how causal models with constraints
can capture many situations of interest.  

We are not the first to suggest moving beyond standard causal models.
In many ways, our framework can be seen as formalizing the 
informal suggestions of Woodward \cite{Woodward15a}.
In addition, Blom, Bongers, and Mooij \cite{BBM19}  
consider \emph{causal constraint models}, which also allow  
non-causally related variables, but their emphasis
lies on extending causal models with additional {\em causal} 
  constraints, rather than the non-causal constraints that we
  consider. (Concretely, they focus 
exclusively on causal representations of 
dynamic systems, and consider the constraints that arise in equilibrium.)
Our work
differs from theirs in several respects (see Section~\ref{sec:discussion}); the approaches can be
viewed as complementary.

    The rest of this paper is structured as follows. 
The next section reviews the formalism of causal models.
Section~\ref{sec:constraints} introduces our  
new formalism for representing non-causal constraints.
In Section~\ref{sec:axioms}, we provide a sound and complete
axiomatization for causal models with constraints, in the spirit of
that provided by Halpern~\cite{Hal20} for causal models.
We conclude with some discussion in Section~\ref{sec:discussion}.

\section{Causal Models}\label{sec:review}

Before getting to the new definitions, we review the standard
definition of a causal model \cite{Hal20,Hal48}
(with a slight modification; see below).
A \emph{causal model} $M$
is a pair $(\S,\F)$, where $\S$ is a
\emph{signature}, which explicitly 
lists the endogenous and exogenous variables  and characterizes
their possible values, and $\F$ defines a set of \emph{structural
  equations},\index{structural equations} relating the values of the
variables.  Formally, a signature $\S$\index{signature} is a tuple
$(\U,\V,\R)$, where $\U$ is a set of 
exogenous variables,\index{exogenous variable} $\V$ is a set 
of endogenous variables,\index{endogenous variable} and $\R$
associates with every variable $Y \in  
\U \union \V$ a nonempty set $\R(Y)$ of possible values for $Y$ 
(i.e., the set of values over which $Y$ {\em ranges}).  

For some endogenous variables $X \in \V$, $\F$ associates a
function denoted $F_X$ such that $F_X$ maps
$\R(\U \union \V - \{X\})$  to $\R(X)$ (where, if $\vec{Y}$ is a set
of variables, we take $\R(\vec{Y})$ to be an abbreviation for
$\times_{Y \in \vec{Y}} \R(Y)$);
that is, $F_X$
takes as input the values of the variables in $\U \union \V$ other
than $X$,
and returns a value in the range of $X$.
Note that we have departed from standard causal
models \cite{Hal20,Hal48} by not 
requiring $\F$ to associate a function $F_X$ with
\emph{every} variable $X \in \V$, only some of them. This turns out
to be critical when we add constraints.


If the value $F_X$ depends only on the variables in some subset
$\vec{W} \subseteq \U \union \V -\{X\}$, we often write $F_X(\vec{w}) = x$ or
$X = \F_X(\vec{W})$.  For example, if we have an exogenous variable
$U$ and endogenous variables $X_1, \ldots, X_5$, and $X_3$ is the sum
of $X_1$ and $X_2$, we write $X_3 = X_2  + X_1$, omitting $X_4$,
$X_5$, and $U$.  
Formally, if $\vec{Y} =  \U \union \V - \{X\} - \vec{W}$, then $F_X(\vec{w}) = x$
is an abbrevation of 
$F_X(\vec{w},\vec{y}) = x$ for all $\vec{y} \in \vec{Y}$.  
While this shorthand is quite common, as we will see below it is particularly useful in
the presence of constraints.
\commentout{
While this shorthand is quite common, it is particularly useful in
the presence of constraints.
Consider a model where $D$ is diet, $W_P$ is
weight in  
pounds, and $W_K$ is weight in kilograms.
It seems strange to think of
$F_{W_P}$ as a function from the values of both $D$ and $W_K$; 
the value of $D$ would never make a difference, given the
value of $W_K$.  While it works mathematically to write the
function so that 
$W_P$ is independent of $W_K$ (and, formally, that is what we do), it
is hard to 
interpret this; it seems much more natural just to think of $W_P$ as a
function of $D$.  
Moreover, it is unnecessary to write an
additional causal equation for $W_K$; it is far more natural for
$W_K$ to be determined by the logical constraint that relates $W_P$
and $W_K$.
This is exactly why we do not require $F_X$ to be defined for all
variables $X$.
}

Much of the
work on causality has focused on \emph{recursive} or \emph{acyclic}
models, where there are no dependency cycles between variables, and the
values of all endogenous variables are ultimately determined by the
\emph{context}, that is, an assignment of values to the exogenous variables.
As we shall see, once we allow constraints, even in acyclic models,
the values of the endogenous variables may not be determined
by the context; we also need a \emph{state}, that is, an assignment
of values to the endogenous variables.
A context-state pair is called an {\em extended state}.
Our approach for dealing with constraints generalizes the way
Halpern \cite{Hal48} deals with cyclic models, so we allow 
cyclic models from the start.
Given a signature $\S$, let $\MS$
denote all causal models
of the form $(\S,\F)$, where $\F$ can be arbitrary.

It is useful to have a language for reasoning about causality.  The
language that has been used in earlier papers is defined as follows:
Given a signature $\S = (\U,\V,\R)$, a \emph{primitive
  event}\index{primitive event} is a
formula of the form $X = x$, for  $X \in \V$ and $x \in \R(X)$. 
A {\em basic causal formula (over $\S$)\/}\index{causal formula} is
one of the form 
$[Y_1 \gets y_1, \ldots, Y_k \gets y_k] \phi$,
where \begin{itemize} 
\item $\phi$ is a Boolean
combination of primitive events,\index{primitive event}
\item $Y_1, \ldots, Y_k$ are distinct variables in $\V$, and
\item $y_i \in \R(Y_i)$, for $i = 1, \ldots, k$.
\end{itemize}
Such a formula is abbreviated
as $[\vec{Y} \gets \vec{y}]\phi$, using the vector notation.
The special case where $k=0$
is abbreviated as $[\,] \phi$.%
\footnote{In standard acyclic models (where there is no disconnection
and an equation for each endogenous variable), 
  we can identify $[\, ]\phi$ and the formula $\phi$, but in our
  setting, we cannot do so.}
We assume for simplicity that the variables in $\V$ are ordered, and,
no matter in what order the variables appear in an intervention, the
resulting formula is syntactic sugar for the formula where the
variables appear in order.  For example if
$Y_1$ is earlier in the order than $Y_2$, then 
$[Y_2 \gets y_2, Y_1 \gets y_1]\phi$ is syntactic sugar for
$[Y_1 \gets y_1, Y_2 \gets y_2]\phi$.
(This assumption is made implicitly in \cite{GallesPearl98,Hal20,HP21},
the papers that we are aware of that provide axiomatizations for
causal models. Without it, the
axiomatizations they provide would not be complete: we would
need an axiom that allows us to rearrange  the order of interventions.)
Intuitively,
\mbox{$[Y_1 \gets y_1, \ldots, Y_k \gets y_k] \phi$} says that
$\phi$ would hold if
$Y_i$ were set to $y_i$, for $i = 1,\ldots,k$.
A {\em causal formula} is a Boolean combination of basic causal formulas.
For $\S = (\U,\V,\R)$, 
let $\LS$ consist of
all
causal formulas
where the variables in the formulas are taken from $\V$ and 
their possible values are determined by $\R$.

A causal formula $\psi$ is true or false in a causal model, given 
an
extended state.
We write $(M,\vec{u},\vec{v}) \sat \psi$  if
the causal formula $\psi$ is true in
causal model $M$ given 
extended state $(\vec{u},\vec{v})$.
The $\sat$ relation is defined inductively
(see \cite{HP01b,Hal48}).
$(M,\vec{u},\vec{v}) \sat X=x$ if $(\vec{u},\vec{v})$ satisfies all
the equations in $\F$ and $X=x$ in state $\vec{v}$.
We extend $\sat$ to conjunctions and negations in the standard way.
Finally,
$(M,\vec{u},\vec{v}) \sat
[\vec{Y} \gets \vec{y}]\phi$ iff 
$(M_{\vec{Y} \gets \vec{y}},\vec{u},\vec{v}')
\sat \phi$ for all  
states $\vec{v}'$ such that $(\vec{u},\vec{v}')$ satisfies all the
equations in 
 $\F_{\vec{Y} \gets \vec{y}}$, where  $M_{\vec{Y}\gets \vec{y}} =
 (\S,\F_{\vec{Y} \gets \vec{y}})$,   and $\F_{\vec{Y} \gets \vec{y}}$
 is identical to $\F$, except that 
 for each $Y_i$ in $\vec{Y}$ and corresponding $y_i$ in $\vec{y}$,
 the
 causal
equation for $Y_i$ is 
replaced by $Y_i = y_i$ (or $Y_i = y_i$ is added if there was no
equation for $Y_i$ in $\F$).  
We write $(M,\vec{u}) \sat \psi$ if the truth of $\psi$ depends only
on the context $\vec{u}$, which is easily seen to be the case for 
formulas of the form $[\vec{Y} \gets   \vec{y}]\phi$
and write $\vec{v} \sat \psi$ if $\psi$ is a Boolean combination of
primitive events that is true in state $\vec{v}$ (note that the truth
of Boolean combinations of primitive events is completely determined
by the state).

Some comments:
\begin{itemize}
\item In a standard acyclic causal model, 
        there is a unique $\vec{v}$ such that $(\vec{u},\vec{v})$ satisfies
the equations in $\F$.
That is why, in the standard semantics for causal formulas
  in acyclic causal models, there is no mention of the 
  state $\vec{v}$;
  cf. \cite{Hal48}.
In cyclic causal models there may be more than
one such $\vec{v}$ such that $(\vec{u},\vec{v})$ satisfies the equations in
$\F$, or none. 
Once we drop the requirement that there is an equation for each
endogenous variable,
there may again be more than one such $\vec{v}$, even in acyclic
models. 
  \item It is easy to check that this definition is equivalent to the
    standard definition of $\sat$ in acyclic causal models.
  \item If we define $\langle \vec{X} \gets \vec{x} \rangle \phi$ as
    an abbreviation of 
    $\neg [\vec{X} \gets \vec{x}] \neg \phi$, then 
    $(M,\vec{u}) \sat \langle \vec{X} \gets \vec{x} \rangle \true$ iff
        there is some state $\vec{v}$ such that $(\vec{u},\vec{v})$
        satisfies all the 
    causal equations in $\F_{\vec{X} \gets \vec{x}}$. 
In this case we say that $\vec{v}$ is a {\em solution} of
$(M_{[\vec{X} \gets \vec{x}]},\vec{u})$, meaning that (with the
obvious abuse of notation)
$(M,\vec{u}) \sat \langle \vec{X} \gets \vec{x} \rangle \V=\vec{v}$. 
\end{itemize}

\section{Causal Models With Constraints}\label{sec:constraints}
We now extend causal models by allowing constraints.
Some of the constraints we are interested in are defined by equations,
such as $\TOT= \HDL + \LDL$.  But we also want to allow constraints 
such as (1)  $X \le Y$, (2)
$X-Y \in \mathbf{S}$ (where $\mathbf{S}$ is a set
of values), and (3) 
$X$ and $Y$ are either both positive or both negative.
Thus, we take a \emph{causal model with constraints} to be a triple $(\S,
\F, \C)$, where, as before, $\S$ is a signature and $\F$ is a
collection of equations, and $\C$ is a set of 
extended
states (intuitively, the
extended
states that satisfy the constraints).
In the special case where $\C$ contains all possible 
extended
states (i.e., where
$\C = \times_{Z \in \U \cup \V} \R(Z)$, so $\C$ places no constraints)
and $\F$ associates an equation with each variable in $\V$,
the causal model with constraints
$(\S, \F, \C)$ is equivalent to the 
standard
 causal model $(\S, \F)$.
Given a signature $\S$, let $\MSc$
consist of all causal models
with constraints of the form $(\S,\F,\C)$, where $\S$ is fixed and $\F$
and $\C$ are arbitrary.

We give semantics to formulas in $\LS$ just as before, except
that we take $\C$ into account.  Specifically,
$(M,\vec{u}) \sat [\vec{Y} \gets \vec{y}]\phi$ iff 
$(M_{\vec{Y}\gets \vec{y}},\vec{u}, \vec{v}) \sat \phi$ for
all
states $\vec{v}$ such that $(\vec{u},\vec{v}) \in \C$ and
$(\vec{u},\vec{v})$  satisfies all the
causal equations in 
   $\F_{\vec{Y} \gets \vec{y}}$.   

Note that, crucially, the causal equations only matter for 
extended states $(\vec{u},\vec{v})$ that satisfy the constraints. 
This explains why we often
need not write an equation $F_X$ as depending on all other variables,
and why not all endogenous variables require a causal
equation. Consider a model where $D$ is diet, $W_P$ is weight in
pounds, $W_K$ is weight in kilograms, and the constraints $\C$
implement the obvious logical constraint that relates $W_P$ and $W_K$
(meaning that $W_P$ and $W_K$ fully determine each other). It does not
matter whether we write $F_{W_P}$ as a function of the values of
both $D$ and $W_K$ or as a function only of $D$, since for all
extended
states $(w_P,w_K,d)$ and $(w_P,w'_K,d)$, if both
$(w_P,w_K,d)$ and $(w_P,w'_K,d)$ are in $\C$, then 
$w_P=F_{W_P}(w_K,d)$ iff $w_P=F_{W_P}(w_K',d)$. Moreover, it is
unnecessary to write an 
additional causal equation for $W_K$; it is far more natural for
$W_K$ to be determined by the logical constraint that relates $W_P$
and $W_K$. 

We find it useful to extend the language $\LS$ a little further, to
allow us to
disconnect
some variables $\vec{X}$ from their causal equations, so that
the values of the variables in $\vec{X}$ are  
determined only by the constraints. Specifically,
we allow formulas of the form $[\disc(\vec{X}), \vec{Y} \gets
  \vec{y}]\phi$, where $\vec{X}$ and $\vec{Y}$ 
  are disjoint, and either of $\vec{X}$ or $\vec{Y}$ may be empty.
  $(M,\vec{u)} \sat  [\disc(\vec{X}),\vec{Y} \gets \vec{y}]\phi \mbox{ iff }
    (M_{-\vec{X}},\vec{u}) \sat [\vec{Y} \gets \vec{y}]\phi$, 
    where $M_{-\vec{X}}$ is the
model \fullv{that is} just like $M$, except 
	that all causal equations for
    variables in $\vec{X}$ are removed from $\F$.%
\footnote{Requiring that $\vec{X}$ and $\vec{Y}$ be
disjoint does not lose expressive power. If 
$\vec{X}$ and $\vec{Y}$ were not disjoint,
we would want
$(M,\vec{u} \sat [\disc(\vec{X}),\vec{Y} \gets \vec{y}]\phi$ iff
        $(M,\vec{u}) \sat 
        [\disc(\vec{X} - \vec{Y}),\vec{Y} \gets \vec{y}]\phi$.}
Let $\LSd$ be the language that extends $\LS$ by allowing disconnection.

Causal models with constraints, as the name suggests, extend causal
models by adding constraints on possible solutions to the structural
equations.  While, at some level, this is a straightforward extension,
as the examples we present below show, it actually adds
significant expressive power, letting us capture
realistic situations that 
cannot be 
 captured in standard causal models.  The extension
also brings out some subtle issues regarding the relationship between
exogenous and endogenous variables and how the value of an endogenous
variable is determined that we briefly discuss here.
\begin{itemize}
\item In some respects, an endogenous variable for which
there is no equation behaves similarly to an exogenous variable: neither
is determined by the structural equations, and they can both
be restricted by the constraints. However, in other respects, 
they behave quite differently: the value of an exogenous
variable is assumed to be simply given, as it's determined by factors
that are not part of our model, whereas the value of an endogenous
variable that does not have an equation is either free to take on any
value that is allowed by the constraints, or is set to some value by
means of an intervention.  
\item We could further generalize the way that the values of
endogenous variables are determined. Instead
of having to choose between an endogenous variable $X$ being uniquely
determined by its equation or not being determined by an equation at
all, we could have an equation $F_X$ such that $F_X$ maps
$\R(\U \union \V - \{X\})$ to 
$\cal{P}(\R(\mbox{$X$}))$. (Peters and
Halpern 
\cite{PH21} go even further and abandon equations altogether,
taking a causal model to simply be a mapping from context-intervention
pairs to states.) Although we believe that this is a sensible generalization,
we believe that the current framework is already sufficiently expressive to
merit a discussion of its own. 
\end{itemize}

\xam\label{xam:temperature}
Suppose that two different researchers 
study the effect of temperature on heat stroke in vulnerable populations. 
One
expresses temperature in Celsius (and uses a variable $TC$
to represent temperature in Celsius) while 
the other
uses a
variable $TF$ to represent temperature in Fahrenheit.  We can 
combine 
their two models into a single causal model $M$ that includes the
constraint $TF = 1.8TC + 32$ (which
means that $\C$ consists of all those 
extended
states where the equation
holds).  For simplicity, suppose that the value of $TC$ is determined
by an exogenous variable $U$
 according to the causal equation
 $TC = U$. 
  There is no causal equation for $TF$ (whose value is determined 
by the constraint). 
There is one other
variable $HS$ (the patient will suffer heatstroke), 
with the causal equation 
$HS = 1$ if
$TC \ge 40$, and $HS = 0$ otherwise.  Consider the context $u$ where $U=35$, so
that $TC = 35$, $TF = 95$, and $HS=0$.  Clearly we have 
that $(M,\vec{u}) \sat \langle TC \gets 40 \rangle (HS=1)$;
  if we set $TC$ to 40,
there is a unique solution to the equations, and in that solution $HS=1$.
On the other hand, we do not have
$(M,\vec{u}) \sat \langle TF \gets 104 \rangle (HS=1)$.
 If we set $TF$ to 104 degrees in context
$u$, then $TC$ remains at 35 degrees (since the value of TC is
determined by the context $u$, which has not changed).  The resulting
state is not in $\C$; there are no solutions to the equations in $\C$
where $TC = 35$ and $TF = 104$.  Thus, $[TF \gets 104](HS=0)$ is
vacuously true in all these solutions; that is, 
$(M,\vec{u}) \sat [TF \gets 104](HS=0)$.
On the other hand, we have
$(M,\vec{u}) \sat \langle \disc(TC), TF \gets 104 \rangle (HS=1)$.
  Once we disconnect
the equation for $TC$, there is a (unique) solution to the equations
where $TF = 104$; in that solution, $TC = 40$ (because of the
constraint) and $HS=1$.  The key point here is that we need to
disconnect $TC$ to get the desired effect of intervening on $TF$.
\exam

Now consider a formalization of the cholesterol example.
\xam \label{cholesterol} 


\commentout{
Consider a model $M$ that represents the impact of cholesterol
on atherosclerosis in a particular patient. While it is normal for physicians to report
total cholesterol level, total cholesterol includes two different kinds 
of cholesterol: high density lipoproteins (HDL cholesterol) and 
low density lipoproteins (LDL cholesterol).\footnote{Total cholesterol 
also includes very low density lipoproteins (VLDL cholesterol). In practice,
VLDL is usually estimated as a fraction of total triglycerides. We
ignore  
these details to simplify the exposition.} LDL cholesterol is harmful, contributing 
to the buildup of plaque in arteries.
By contrast, HDL cholesterol is beneficial, since it helps to clear LDL cholesterol out of the arteries. 
The model has the following
endogenous variables:
\begin{itemize}
  \item $AS$ -- atherosclerosis, level of plaque build-up in arteries
  \item $\HDL$ -- level of HDL cholesterol
  \item $\LDL$ -- level of LDL cholesterol
  \item $\TOT$ -- total cholesterol
      \item $D$ -- dietary factors that affect cholesterol.
\end{itemize}
There is one exogenous variable, $U$.
The causal equations are
 \begin{itemize}
  \item $D = F_D(U)$
  \item $\HDL = F_{\HDL}(D)$
  \item $\LDL = F_{\LDL}(D)$
  \item $AS = F_{AS}(\HDL, \LDL)$
\end{itemize}
$D$ is determined by the exogenous variable (i.e., the context).
We do not specify the precise equation, but assume that
$F_{AS}$ is a decreasing function of $\HDL$ and increasing in $\LDL$.
The constraints $\C$ consists of 
all the states where $\TOT = \HDL + \LDL$. 

In this model, we can freely intervene on $\HDL$ and $\LDL$; the value
of $\TOT$ will change in the appropriate way, so as to maintain the
constraint.  Of course, if we intervene to set  $\HDL = \hdl$, $\LDL =
\ldl$, and $\TOT = \tot$ 
simultaneously, then unless the intervention is such that $\tot = \ldl +
\hdl$, there will be no 
states satisfying the constraints, so all
formulas of the form $[\LDL \gets \ldl, \HDL \gets \hdl, \TOT \gets
  \tot]\phi$ will be vacuously true.  Indeed, in a context $\vec{u}$
where $\LDL = \ldl^*$, $\HDL = \hdl^*$, and $\TOT = \tot^*$, an intervention
that sets $\TOT$ to $\tot' > \tot^*$ will also lead to an
inconsistency, unless we disconnect the equation for either $\LDL$ or
$\HDL$ (or both).

Note that if we disconnect the equation for $\LDL$
(but not the equation for $\HDL$),
there will be a unique solution to the equations, where $\HDL = \hdl^*$,
$\TOT = \tot'$, and $\LDL = \tot' - \hdl^*$.  That is, intervening on $\TOT$
while disconnecting $\LDL$ results in the value of $\HDL$ remaining
fixed, while $\LDL$ changes to maintain the constraint.  Similarly, if
we disconnect $\HDL$ but not $\LDL$.  If we disconnect both $\LDL$ and
$\HDL$ while setting $\TOT = \tot'$, then there will be multiple
solutions to the equations: $\HDL$ and $\LDL$ can take arbitrary values
that add up to $\tot'$.
This makes $(\disc(\LDL,\HDL), \TOT=\tot')$ what Spirtes and Scheines
\cite{SS04} call an \emph{ambiguous intervention}.
}
Consider a model $M$ that represents the impact of cholesterol
on atherosclerosis in a particular patient. While it is normal for physicians to report
total cholesterol level, total cholesterol includes three different kinds 
of cholesterol: HDL cholesterol, (LDL cholesterol, and very low-density
lipoproteins (VLDL cholesterol). LDL cholesterol is harmful,contributing  
to the buildup of plaque in arteries.
By contrast, HDL cholesterol is beneficial, since it helps to clear LDL cholesterol out of the arteries. 
VLDL cholesterol has little direct impact on the arteries, but it contributes to levels of triglycerides, 
which are harmful.
In practice, it is very difficult to directly measure LDL and VLDL cholesterol. Instead, VLDL cholesterol is 
inferred from observed triglyceride levels, and this inferred value is used together with
measured values of HDL and total cholesterol to estimate the value of LDL cholesterol. 
For this reason, it may be useful to be able to include all of these variables together in the same
causal model. 
The model has the following
endogenous variables:
\begin{itemize}
  \item $AS$ -- atherosclerosis, level of plaque build-up in arteries
  \item $\HDL$ -- level of HDL cholesterol
  \item $\LDL$ -- level of LDL cholesterol
  \item $\VLDL$ -- level of VLDL cholesterol
  \item $\TOT$ -- total cholesterol level
  \item $\TRI$ -- level of triglycerides
      \item $D$ -- dietary factors that affect cholesterol.
\end{itemize}
There is one exogenous variable, $U$.
The causal equations are
 \begin{itemize}
  \item $D = F_D(U)$
  \item $\HDL = F_{\HDL}(D)$
  \item $\LDL = F_{\LDL}(D)$
  \item $\VLDL = F_{\VLDL}(D)$
  \item $\TRI = F_{\TRI}(\VLDL)$
  \item $AS = F_{AS}(\HDL, \LDL, \TRI)$
\end{itemize}
$D$ is determined by the exogenous variable (i.e., the context).
We do not specify the precise equations, but assume that 
$F_{AS}$ is a decreasing function of $\HDL$ and increasing in $\LDL$
and $\TRI$; we also assume that $F_{\TRI}$ is an increasing function of 
$\VLDL$.
The constraint $\C$ consists of 
all the states where $\TOT = \HDL + \LDL + \VLDL$. 

In this model, we can freely intervene on $\HDL$, $\LDL$, and $\VLDL$; the value
of $\TOT$ will change in the appropriate way, so as to maintain the
constraint.  Of course, if we intervene to set  $\HDL = \hdl$, $\LDL =
\ldl$, $\VLDL =
\vldl$ and $\TOT = \tot$ 
simultaneously, then unless the intervention is such that $\tot = \ldl +
\hdl + \vldl$, there will be no 
states satisfying the constraints, so all
formulas of the form $[\LDL \gets \ldl, \HDL \gets \hdl, \VLDL \gets \vldl, \TOT \gets
  \tot]\phi$ will be vacuously true.  Indeed, in a context $\vec{u}$
where $\LDL = \ldl^*$, $\HDL = \hdl^*$, $\VLDL = \vldl^*$, and $\TOT = \tot^*$, an intervention
that sets $\TOT$ to $\tot' > \tot^*$ will also lead to an
inconsistency, unless we disconnect the equation for at least one of $\LDL$,
$\HDL$, or $\VLDL$.

In the context just described, 
if we intervene to set $\TOT = \tot'$, while disconnecting the equation for $\LDL$
(but not the equations for $\HDL$ and $VLDL$),
there will be a unique solution to the equations, where $\HDL = \hdl^*$,
$\VLDL = \vldl^*$,
$\TOT = \tot'$, and $\LDL = \tot' - \hdl^* - \vldl^*$.  That is, intervening on $\TOT$
while disconnecting $\LDL$ results in the values of $\HDL$ and $\VLDL$ remaining
fixed, while $\LDL$ changes to maintain the constraint.  Similarly, if
we disconnect only $\HDL$ or only $\VLDL$.  If we disconnect all of 
$\HDL$, $\LDL$, and $\VLDL$ 
while setting $\TOT = \tot'$, then there will be multiple
solutions to the equations: $\HDL$, $\LDL$, and $\VLDL$ can take arbitrary values
that add up to $\tot'$.
This makes $(\disc(\LDL,\HDL, \VLDL), \TOT=\tot')$ what Spirtes and Scheines
\cite{SS04} call an \emph{ambiguous intervention}.
\exam

The next example shows that using the disconnection operation allows
us to distinguish different ways of implementing an intervention on a variable.
\xam\label{geometry}
A point is confined to the first quadrant of the Cartesian plane. 
We can represent its position using Cartesian coordinates $X$ and $Y$, with $0 < X, Y$.
We can also represent its position using polar coordinates $R$ and $\theta$, with
$0 < R$ and $0 < \theta < \frac{\pi}{2}$.
The  model 
requires
$X, Y, R,$ and $\theta$ to satisfy the usual constraints:
\begin{itemize}
\item $R = \sqrt{X^2 + Y^2}$
  \item $\theta = \arctan(\frac{Y}{X})$.
\end{itemize}
In the absence of intervention, the point will remain in place. Thus, we can 
have as our exogenous variable the previous position of the point $U =
(U_X, U_Y)$. 
The causal equations are
\begin{itemize}
\item $X = U_X$
\item $Y = U_Y$
\item $R = \sqrt{U_X^2 + U_Y^2}$
  \item $\theta = \arctan(\frac{U_Y}{U_X})$.
\end{itemize}

%
If we want to set the value of $X$ in a meaningful way, we need to either
disconnect both $R$ and $\theta$ or disconnect $Y$ and $R$.  That is, 
we consider interventions of the form
\begin{itemize}
  \item $\disc(R, \theta), X \gets x$ and 
  \item $\disc(Y, R), X \gets x$.
\end{itemize}
The first intervention sets the value of $X$ while leaving $Y$
alone; technically, this means that $Y$ takes the value determined by the causal
equations. This corresponds 
to sliding the point horizontally until the desired value of $X$ is reached. This intervention
removes $R$ and $\theta$ from the influence of 
their causal equations,
effectively forcing them to take the values determined by the constraints. 
The second intervention, $\disc(Y, R), X \gets x$, sets the value of $X$ 
while leaving $\theta$ alone. This corresponds to sliding the point along the ray
connecting its current position to the origin, until the desired value
of $X$ is reached.
%
 We can also consider the intervention
 $disc(Y, \theta), X \gets x$.
 This corresponds to rotating the point around the origin until $X = x$.
In context $(u_X, u_Y)$, 
this intervention only yields
solutions consistent with the constraint 
when
$x < \sqrt{u_X^2 + u_Y^2}.$
%
\exam

\section{A sound and complete axiomatization for causal models with
constraints}\label{sec:axioms}

In this section we provide a sound and complete axiomatization for
the language $\LSd$ with respect to $\MSc$. 
Following \cite{Hal20}, we restrict to the case that $\S = (\U,V,\R)$
is finite, that is, $\U$ is finite, $\V$ is finite, and $\R(X)$ is finite
for all $X \in \U \cup \V$.

Halpern \cite{Hal20} considers a somewhat richer language than we do,
where the context $\uu$ is part of the formula, 
not on the left-hand side of the $\models$.  Specifically, Halpern
considers primitive events of the form $X(\vec{u}) = x$, where
$M \sat X(\vec{u}) = x$ in Halpern's semantics iff 
$(M,\vec{u}) \sat X = x$ in our semantics.  We follow what is now the
more standard usage, with the context $\vec{u}$ on the left of $\sat$.
We thus follow \cite{HP21} and consider a variant of Halpern's 
axioms more appropriate for our language.
\commentout{
Before we review these
axioms, we need to recall some notation.  To capture acylicity,
Halpern defined the formula $Y \leadsto Z$, read ``$Y$ affects $Z$'',
as an abbreviation for the formula
\begin{equation} \nonumber
	\begin{split}
	  &\lor_{\XX \subseteq \V,
\xx \in \R(\XX), Y \notin \XX,
y \in \R(y), z \ne z' \in \R(Z)} \\
&\quad (\<\XX \gets \xx\>(Z = z) \land \<\XX \gets \xx,
		Y \gets y\>(Z = z'));%
                \footnote{In \cite{Hal20,HP21}, the corresponding
                formula used $[\cdot]$ rather than $\<\cdot\>$.  in
                acyclic models without constraints, the two choices
                are equivalent.  Once we add constraints, they are no
                longer equivalent and, for our purposes, $\<\cdot\>$
                is more appropriate.}
	\end{split}
\end{equation}
that is, $Y$ affects $Z$
if there is some setting of some endogenous variables $\XX$
for which changing the value of $Y$ changes the
value of $Z$ (see axiom D6 below).    
}

Here are Halpern's axioms, as given in \cite{HP21} (we keep the same
numbering):%
\footnote{The axiom D6 that we omit is for axiomatizing acyclic
models, since our focus is on general models here.}
\begin{itemize}
	\item[D0.] All instances of propositional tautologies.
	\item[D1.] $[\YY \gets \yy](X = x \rimp
		      X \ne x')$  if $x, x' \in \R(X)$, $x \ne x'$ \hfill
	\item[D2.] $[\YY \gets \yy](\bigvee_{x \in \R(X)} X = x)$
\fullv{\hfill (definiteness)}
	\item[D3.] $\<\XX \gets \xx\>(W = w
		      \land \phi)
		      \rimp \<\XX \gets \xx, W \gets
		      w\>(\phi)$ if $W \notin \XX$%
                      \footnote{The requirement $W \notin \XX$ is not
explicit in \cite{Hal20}, but is needed to ensure that the
	variables in $\<\XX \gets \xx, W \gets w\>$ are distinct.}
	      \hfill
\fullv{(composition)}
	\item[D4.] $[\XX \gets \xx](\XX = \xx)$ \hfill
\fullv{(effectiveness)}
	\item[D5.] $(\<\XX \gets \xx, Y \gets y\> (W = w \land
		      \ZZ = \zz)  \land
		      \<\XX \gets \xx, W \gets w\> (Y = y \land
\ZZ = \zz))$\\
	      $\mbox{ }\ \ \ \rimp \<\XX \gets \xx\> (W
		      = w \land Y = y \land \ZZ = \zz)$ if $\vec{Z}
		      = \V - (\vec{X} \cup \{W,Y\})$
\mbox{ } \fullv{ \hfill (reversibility)}
	\item[D7.] $([\XX \gets \xx]\phi \land [\XX \gets
			      \xx](\phi \rimp \psi)) \rimp  [\XX \gets \xx]\psi$
\fullv{\hfill 	      (distribution)}
	\item[D8.] $[\XX \gets \xx]\phi$ if $\phi$ is a propositional
tautology  \fullv{\hfill (generalization)}
	\item[D9.]
	      $\<\YY \gets \yy\>true \wedge (\<\YY \gets \yy\>\phi
		      \Rightarrow [\YY\gets \yy]\phi)$
	      \ if $\YY= \V$ or, for some $X \in \V$, 
	      $\YY = \V - \{X\}$
\fullv{\hfill           (unique        	      outcomes for $\V$ and
$\V - \{X\}$)%
\footnote{Halpern \cite{Hal20} did not include
	the case that $\YY = \V$, but it seems necessary for completeness.}
}
	\item[MP.] From $\phi$ and $\phi \rimp \psi$, infer $\psi$
\fullv{\hfill (modus ponens)}
\end{itemize}
We refer the reader to \cite{HP21} for a detailed discussion of how these
axioms compare to those of Halpern \cite{Hal20}.

Let $\AXp$ consist
of axiom schema D0-D5 and D7-D9, and inference rule MP.

\begin{theorem}\label{thm:completeness-for-SEMs}
   {\rm  \cite{Hal20}}
     $\AXp$
     is a sound and complete axiomatization for
	the language $\LS$ with respect to $\MS$.
\end{theorem}

We now want to extend this result to causal models with constraints.
The first step is to deal with disconnection, which can be done using
the following surprisingly simple axiom, where $\R(\vec{X})
= \times_{X \in \vec{X}} \R(X)$.

\begin{itemize}[leftmargin=15pt, align=left,
labelwidth=1\parindent, labelsep=5pt]
\item[{DSC.}]  $[\disc(\vec{X}), \vec{Y} \gets
  \vec{y}]\phi \dimp \bigwedge_{\vec{x} \in \R(\vec{X})}
  [\vec{X} \gets \vec{x}, \vec{Y} \gets   \vec{y}]\phi.$
\end{itemize}
Roughly speaking, DSC says that disconnecting all the variables in
$\vec{X}$ is the same as nondeterministically assigning the variables
in $\vec{X}$ an arbitrary value in their range.  As we shall see, DSC
is exactly what we need to capture disconnection.

We also need to modify D9. We break the modification up into
two parts, which we discuss further below.
\begin{itemize}[leftmargin=15pt, align=left,
labelwidth=1\parindent, labelsep=5pt]
\item[D9$'$.]
$(\<\YY \gets \yy\>(X=x) \land
\<\YY \gets \yy\>(X=x') \land
\<\YY \gets \yyp,X\gets x''\>\true) \rimp \<\YY \gets \yyp\>(X=x'')$
if $\YY = \V - \{X\}$ and $x \ne x'$.
\item[D9$''$.] $\land_{x \in \R(X)} \<\YY \gets \yy, X \gets
x\>\true \rimp \<\YY \gets \yy\> \true$, where $\YY = \V-\{X\}$.
\end{itemize}
D9$'$ is intended to deal with the case that $F_X$ is
undefined (i.e., $\F$ does not associate a function $F_X$ with the
variable $X$) in a causal model $M$.  This must be the case if there
are two distinct values $x, 
x' \in \R(X)$ such that $(\<\YY \gets \yy\>(X=x) \land
\<\YY \gets \yy\>(X=x')$ is true in $(M,\vec{u})$ for some context
$\vec{u}$.  In that case, 
$\<\YY \gets \yyp\>(X=x'')$ must be true in $(M,\vec{u})$ for all
$\yyp \in \R(\vec{Y})$ and $x'' \in \R(X)$ such that
$(\vec{u},\yyp,x'') \in \C$, which will be the case exactly if
$\<\YY \gets \yyp,X\gets x''\>\true$ is true in $(M,\vec{u})$.
D9$''$ says that, for a fixed setting $\vec{y}$ of the variables in
$\vec{Y} = \V-\{X\}$, if the constraints do not preclude $X$ from
taking any value, then there is some solution to the equations
$\F_{\vec{Y} \gets \vec{y}}$, whether or not $F_X$ is defined.


Let $\AXpd$
be the result of adding axiom DSC 
to $\AXp$ and replacing D9 by D9$'$ and D9$''$.

\begin{theorem}
     $\AXpd$
     is a sound and complete axiomatization for
	the language $\LSd$ with respect to 
$\MSc$.
\end{theorem}

\begin{proof}
\commentout{
We prove completeness here, focusing  on the parts of the proof that
differ from that of \cite{Hal20}. 
We provide the parts of the soundness argument that differ from earlier work
in the appendix.
}
We here focus on the parts of the proof that
differ from that of \cite{Hal20}. 

For completeness, using DSC, we can
eliminate all occurrences of $\disc(\vec{X})$ from formulas, so it
suffices to show that if a formula $\phi \in \LS$ is valid in $\MSc$,
then it is provable from $\AX'$, where $\AX'$ is identical to $\AXp$
except that D9 is replaced by D9$'$ and D9$''$.  
The steps of the argument are standard: It suffices to show that if a formula 
$\phi \in \LS$ is consistent with respect to $\AX'$ (i.e., we cannot prove
$\neg \phi$ in $\AX'$), then there is a causal model with constraints
$M \in \MSc$ and 
a context $\vec{u}$ such that
$(M,\vec{u}) \sat \phi$.

We extend $\{\phi\}$ to a maximal set $C$ of formulas consistent with
$\AX'$.  We 
then use the formulas in $C$ to define a model $M =
(\S,\F,\C) \in \MSc$ such that
in all contexts $\vec{u}$ of $M$ and for all formulas $\psi \in \LS$,
we have that $(M,\vec{u}) \sat \psi$ iff  
$\psi \in C$.  Halpern \cite{Hal20} used the formulas in $C$ to
define $\F$, by taking $F_X(\vec{u},\vec{y}) = x$ if
$\vec{y} \in \R(\V - \{X\})$ and 
$\<\vec{Y} = \vec{y}\>(X=x) \in C$.
It follows easily from D1, D2, and D9 that $F_X$ is well defined: there is a
unique value $x \in \R(X)$ such that $\<\vec{Y}
= \vec{y}\>(X=x) \in C$ for $\vec{Y} = \V-\{X\}$.  We must work harder
here, since we do not 
have axiom D9, only axioms D9$'$ and D9$''$.
For each variable $X \in \V$, there may be a unique
$x \in \R(X)$ such that $\<\vec{Y}
\gets \vec{y}\>(X=x) \in C$, but there
may not be any such value $x$, and there may be more than one.
We have to define $\F$ in all these cases.

We proceed as follows.
We define $\C$ to consist of all extended states $(\vec{u},\vec{v})$
such that $\<\V 
\gets \vec{v}\>\true \in C$.  
To define $\F$, for each variable $X \in \V$ we consider three cases. 
Given $X$, if for some
$\vec{y} \in \vec{Y} = \V -\{X\}$ there are two values $x$ and $x'$ in
$\R(X)$
such that both
$\<\vec{Y} \gets \vec{y}\>(\vec{X} = \vec{x}) \in C$ and
$\<\vec{Y} \gets \vec{y}\>(\vec{X} = \vec{x}') \in C$, then $F_X$ is
undefined.  Otherwise, for all $\vec{y} \in \R(\vec{Y})$, there is at
most one $x \in \R(X)$ such that $\<\vec{Y} \gets \vec{y}\>(X = x) \in
C$.  Thus, 
if there is some $x \in \R(X)$ such that
$\<\vec{Y} \gets \vec{y}\>(X = x) \in C$,
then $x$ is unique, and 
we take
$F_X(\vec{u},\vec{y}) = x$ for all contexts $\vec{u}$.
Finally, if there are no values $x \in
\R(X)$ such that $\<\vec{Y} \gets \vec{y}\>(X = x) \in C$, then
there must be some $x \in \R(X)$ such that
$\<\vec{Y} \gets \vec{y}, \vec{X} \gets \vec{x}\>\true \notin C$, for
otherwise, by  
D9$''$, $\<\vec{Y} \gets \vec{y}\>\true \in C$, and it follows by
standard modal reasoning, using D2, D7, D8, and MP, that
$\<\vec{Y} \gets \vec{y}\>(X = x) \in C$ for some $x \in \R(X)$.
We define
$F_X(\vec{u},\vec{y}) = x$ for all contexts $\vec{u}$.  (If there is more than one value $x$ such 
that $\<\YY \gets \yy, X \gets x\>\true  \notin C$, we can choose one
arbitrarily.)  Let $M = (\S,\F,\C)$, for this definition of $\F$ and
$\C$.  

Since $F_X$ (if it is defined) is independent of
$\vec{u}$, it follows that for all formulas $\psi \in \LS$, 
$(M,\vec{u}) \sat \psi$ for some context $\vec{u}$ iff
$(M,\vec{u}) \sat \psi$ for all contexts $\vec{u}$.
We show that for all
$\psi \in \LS$, we have that 
$(M,\vec{u}) \sat \psi$ for some (and hence all) contexts $\vec{u}$
iff $\psi \in C$.
Using standard modal reasoning as
in \cite{Hal20}, it suffices to consider primitive events and formulas
of the form $\<\vec{Y} \gets \vec{y}\>(\vec{X} = \vec{x})$.  Using D4, we
can further restrict to the case where $\vec{X}$ and $\vec{Y}$ are
disjoint.  We proceed by induction on $|\V-\vec{Y}|$.  If  $|\V
- \vec{Y}| = 0$, then $\vec{Y} = \V$ and we can take $\vec{X}
= \vec{x}$ to be the formula $\true$
and take $\vec{Y} = \vec{y}$ to be
$\V = \vec{v}$ for some state $\vec{v}$.  Note that  
$\<\V \gets \vec{v}\>\true \in C$ iff $\vec{v} \in \C$ iff
$(M,\vec{u}) \sat \<\V \gets \vec{v}\>\true$, as desired.

If $|\V-\vec{Y}| = 1$, then $\V - \vec{Y} = \{X\}$ for some variable
$X \in \V$.
Suppose that $\<\vec{Y} \gets \vec{y}\>(X = x) \in C$.
Then by D3, we must have $\<\vec{Y} \gets \vec{y}, X \gets x\>\true \in C$.
There are two cases: If for some $\yyp \in \R(\vec{Y})$ there
exist two values $x'$ and $x''$ in $\R(X)$  such
that both $\<\vec{Y} \gets \yyp\>(X=x') \in C$ and
$\<\vec{Y} \gets \yyp\>(X=x'') \in C$, then $F_X$ is undefined.
It easily follows that $(M,\vec{u}) \sat
\<\vec{Y} \gets \vec{y}\>(X = x)$.  Otherwise, for all
$\yyp \in \R(\vec{Y})$, there is at most one $x'$ such that
$\<\vec{Y} \gets \yyp\>(X=x') \in C$, so
$x$ has to be the unique value $x' \in \R(X)$ such that
$\<\vec{Y} \gets \vec{y}\>(X = x') \in C$;
therefore, by construction,
$F_X(\vec{u},\vec{y}) = x$.
It again follows that $(M,\vec{u}) \sat 
\<\vec{Y} \gets \vec{y}\>(X=x)$.

For the opposite direction, suppose that
$(M,\vec{u}) \sat \<\vec{Y} \gets \vec{y}\>(X=x)$. Then
$(M,\vec{u}) \sat \<\vec{Y} \gets \vec{y}, X \gets x\>\true$, so
$\<\vec{Y} \gets \vec{y},X \gets x\>\true \in C$ by the induction
hypothesis,  and
either (1) $F_X(\vec{u},\vec{y}) = x$ or (2)
$F_X$ is undefined.  In case (1), by
construction, $\<\vec{Y} \gets \vec{y}\>(X=x) \in
C$.  In case (2), there must be two values $x'$
and $x''$ in $\R(X)$
and some value $\yyp \in \R(\vec{Y})$ such that
such that $\<\vec{Y} \gets \yyp\>(X=x') \in C$
and $\<\vec{Y} \gets \yyp\>(X=x'') \in C$.  Since
$\<\vec{Y} \gets \vec{y},X \gets x\>\true \in C$, by D9$'$,
$\<\vec{Y} \gets \vec{y}\>(X=x) \in C$, as desired.

The inductive step proceeds just as in \cite{Hal20}, using
D3 and D5; we omit the details here.

    We now prove the soundness of $\AXpd$. Given 
$M = (\S,\F,\C) \in \MSc$, we want to show that all the axioms
are valid in $M$.
The argument for D0-D5, D7, and D8
is much like that given in \cite{GallesPearl98,Hal20}; we leave the
details to the reader.  
\commentout{
suppose that $M = (\S,\F,\C) \in \MSc$.
By Theorem~\ref{thm:completeness-for-SEMs}, we know that D0-D5 and D7-D9
are valid in models where there are no constraints and all the
equations are defined.  
So if $M' = (S,\F')$ is a \emph{standard extension of $M$}, that is, a
causal model that results when we remove all constraints from 
$M$ and extend $\F$ to $\F'$ so that $F'_X$ is defined for all
variables $X$, 
then these axioms are valid.
Clearly D0 continues to be valid in $M$.   
It is also easy to see that all axioms 
$\phi$ of the form $[\vec{Y} \gets \vec{y}]\phi'$, which is the case
for D1, D2, D4, and D8,
are valid.  If 
$(\vec{u},\vec{v})$ is a pair that satisfies the 
equations in $\F_{\vec{Y} \gets \vec{y}}$ such that $v \in C$, then
there is a standard extension $M'$ of $M$ such that
$(\vec{u},\vec{v})$ satisfies the equations in
$\F'_{\vec{Y} \gets \vec{y}}$, so we must have $v \sat \phi'$, as
desired.  Thus, D1, D2, D4, and D8 are valid.

For D3, the soundness proof shows that for all $\vec{u}$, if there is a
$\vec{v}$  such that $(\vec{u},\vec{v})$ satisfies the equations in
$\F_{\vec{X} \gets \vec{x}}$ and $\vec{v} \sat
W=w \land \phi$, then for the same choice of $\vec{v}$, we have that
$(\vec{u},\vec{v})$ satisfies the equations in
$\F_{\vec{X} \gets \vec{x}, W \gets w}$. 
This continues to be the case if $\vec{v} \in \C$, so D3 is still
sound in the presence of constraints.  Essentially the same argument
shows that D5 and D7 continue to be sound in the presence of
constraints.
}

For D9$'$, observe that without constraints, D9 is sound
because for each $\vec{u}$, there is a unique solution $\vec{v}$ to
the equations in $\F_{\vec{Y} \gets \vec{y}}$ if $\vec{Y}$ consists
of all but one 
endogenous variable.  With constraints, there may not be
a solution at all (so the first conjunct of D9 is not sound), 
and there may be many solutions if  
$F_X$ is undefined.
If $(M,\vec{u},\vec{v}) \sat
(\<\YY \gets \yy\>(X=x) \land
\<\YY \gets \yy\>(X=x') \land
\<\YY \gets \yy,X\gets x''\>\true)$,
(with $x \neq x')$,
then there are at least two solutions to the equations ($x$ and $x'$),
so $F_X$ must be undefined.  That means that if
$(\vec{u},\yyp,x'') \in \C$, which must be the case if 
$(M,\vec{u}) \sat
\<\YY \gets \yyp,X\gets x''\>\true$, then $(M,\vec{u},\vec{y}) \sat 
\<\YY \gets \yyp\>(X=x'')$, as desired. 

For D9$''$, suppose that 
$(M, \vec{u}) \sat \land_{x \in \R(X)} \<\YY \gets \yy, X \gets
x\>\true$.  We want to show that
$(M, \vec{u}) \sat \<\YY \gets \yy\> \true$.  Suppose that
$F_X$ is defined and
$F_X(\vec{u}, \vec{y}) = x$.  Let $\vec{v}$ be such that 
$\vec{v} \sat \YY = \yy \land X=x$.  Since
$(M,\vec{u}) \sat \<\YY \gets \yy, X \gets x\>\true$, and $\vec{v}$ is
the unique state such that $(\vec{u},\vec{v})$ satisfies the equations
in $\F_{\YY \gets \yy, X \gets x}$, it must be the case that
$(\vec{u},\vec{v}) \in \C$ and satisfies the equations in $\F_{\YY \gets \yy}$.
Thus, $(M,\vec{u}) \sat  \<\YY \gets \yy\> \true$, as desired.
On the other hand, if $F_X$ is undefined,
since $(M, \vec{u}) \sat \land_{x \in \R(X)} \<\YY \gets \yy, X \gets
x\>\true$, it must be the case that 
$(\vec{u}, \vec{y},x) \in \C$ for all
$x \in \R(X)$, and 
$(\vec{u},\vec{y},x)$ satisfies all the equations in
$\F_{\vec{Y} \gets \vec{y}}$, so
$(M,\vec{u}) \sat \bigwedge_{x \in \R(X)} \<\YY \gets \yy\>(X=x)$,
and hence $(M,\vec{u}) \sat \<\YY \gets \yy\>\true$,

Finally, for DSC, suppose that 
$(M,\vec{u}) \sat [\disc(\vec{X}), \vec{Y} \gets
  \vec{y}]\phi$.  Then $(M_{-\vec{X}},\vec{u}) \sat
  [\vec{Y} \gets \vec{y}]\phi$.  So 
for all $(\vec{u},\vec{v}) \in \C$ such that
  $(\vec{u},\vec{v})$ satisfies the equations in
  $\F_{-\vec{X},\vec{Y} \gets \vec{y}}$, we have
  that $\vec{v} \sat \phi$.  We claim that, for all
  $\vec{x} \in \R(\vec{X})$, we have that 
$(M,\vec{u}) \sat
[\vec{X} \gets \vec{x}, \vec{Y} \gets \vec{y}]\phi$.  For suppose that
$(\vec{u},\vec{v}) \in \C$ and 
  $(\vec{u},\vec{v})$  satisfies the equations in
  $\F_{\vec{X} \gets \vec{x},\vec{Y} \gets \vec{y}}$.  Then
  $(\vec{u},\vec{v})$ clearly satisfies the equations in
    $\F_{-\vec{X},\vec{Y} \gets \vec{y}}$, so $\vec{v} \sat \phi$.
The result follows.

Conversely, suppose that  $(M,\vec{u}) \sat 
\bigwedge_{\vec{x} \in \R(\vec{X})}
  [\vec{X} \gets \vec{x}, \vec{Y} \gets   \vec{y}]\phi$.
  We want to show that $(M,\vec{u}) \sat 
  [\disc(\vec{X}), \vec{Y} \gets   \vec{y}]\phi$.  Suppose that
  $(\vec{u},\vec{v}) \in \C$ and $(\vec{u}, \vec{v})$ satisfies the equations in
  $\F_{-\vec{X},\vec{Y} \gets \vec{y}}$.  There
  must be some $\vec{x} \in \R(\vec{X})$ such that
  $\vec{v} \sat \vec{X} = \vec{x}$.  It follows that
  $(\vec{u},\vec{v})$ satisfies the equations in 
  $\F_{\vec{X} \gets \vec{x},\vec{Y} \gets \vec{y}}$. 
Since $(M,\vec{u}) \sat
  [\vec{X} \gets \vec{x}, \vec{Y} \gets \vec{y}]\phi$, we must have
  that $\vec{v} \sat \phi$.  Since this is the case for all
  $(\vec{u},\vec{v}) \in \C$ such that $(\vec{u}, \vec{v})$ satisfies the equations in
  $\F_{-\vec{X},\vec{Y} \gets \vec{y}}$, it follows that
  $(M,\vec{u}) \sat    [\disc(\vec{X}), \vec{Y} \gets   \vec{y}]\phi$,
  as desired.
\end{proof}

\section{Discussion and Related Work}\label{sec:discussion}
We have introduced an approach for allowing 
non-causal
constraints in causal
models.
We believe 
that our approach will have applications well beyond those that we
have discussed.  We mention just 
\shortv{some of them here.}
\fullv{some of them here that we hope to address in 
future work.}

First, there has been recent work on
representing causal models at different levels of abstraction
\cite{BH19,Rub17}.  Representing that a (standard) causal model $M_H$
(intuitively, the high-level model) is an abstraction of $M_L$ (the
low-level model) 
\commentout{
(indeed, this is what has been done in
all the papers on the topic), thereby ignoring the important
non-causal relationships between the coarse-grained and fine-grained
variables of the distinct models. By allowing causal models with
constraints, we can model this directly.  
}%
is done using an \emph{abstraction function} that relates the values of
variables in $M_L$ to those in $M_H$.
Models with constraints can easily capture abstraction.  
Given an abstraction function $\tau:\R(\V_L) \rightarrow \R(\V_H)$, 
we can construct a model with constraints $M$ that simply
combines $M_L$ and $M_H$ (both the signatures and the equations), and
let the constraints $\C$ consist of all extended states
$(\vec{u}_L,\tau_{\U}(\vec{u}_L),\vec{v}_L,\tau(\vec{v}_L))$. 

The work on abstraction has two features that are not directly
captured by this map.  First, they include a set of \emph{allowed
interventions}. 
Intuitively, disallowed interventions are not meaningful or cannot be
performed.  Disallowed interventions in a causal model with
constraints can be viewed as ones that do not have a solution.  
However, it seems
useful to have a more systematic understanding of the set of
interventions that are meaningful and will give rise to solutions. 
Second, abstractions have been generalized to the
{\em approximate} case, so that the solutions to the equations in both
causal models may deviate slightly from the abstraction relation
\cite{BEH19}. As such a situation seems more realistic in practice, it
would be good to generalize causal models with constraints in a
similar manner. One way of doing so would be to consider a metric
$d_{\V}(\cdot,\cdot)$ on the range of endogenous variables $\R(\V)$
and consider as solutions of the model all extended states
$(\vec{u},\vec{v})$ that are within $\alpha$ of some
$(\vec{u}',\vec{v}') \in \C$.
Doing this would allow the tools for approximate abstraction to be
carried over to models with constraints.

\commentout{
\fullv{
Second, in the work on
abstraction, although there are no constraints in the models, there is
a set of \emph{allowed interventions}; that is, not all interventions
can be performed.  The intuition for having allowed interventions is
that certain interventions may not be meaningful, or it may not be
possible to perform them.  Considering models with constraints allows
us to deal with allowed interventions in a more principled way.
}}%

Second,
causal discovery algorithms are usually limited to learning
a causal model using just a single dataset. There has been interesting
work on generalizing causal discovery algorithms to overcome this
limitation, meaning they can take advantage of various datasets using
different variables, greatly improving accuracy \cite{TE14,HZGG20}.
This work has not yet considered non-causal relationships between variables. A
natural step to take is to modify these algorithms so that they can
exploit the constraints between variables appearing in different
datasets, and learn a causal model with constraints.


Third, it
 is worth examining the relative expressive power of our
approach and that of Blom et 
al. \cite{BBM19}.  As we said, they also allow non-causally
related variables.  They in fact allow a more general class of
constraints,
ones that are active only under certain
interventions.   However, we allow disconnection (i.e., the $\disc()$
operation), which allows us to remove causal constraints.
\fullv{As we saw in our examples, disconnection plays a critical role; in
particular, 
as Example~\ref{geometry} shows, it allows us to specify how we want
to implement an intervention on a particular variable in a way that we
believe is quite useful in practice.}
\shortv{As Example~\ref{geometry} shows, disconnection allows us to
specify how we want 
to implement an intervention on a particular variable in a way that we
believe is quite useful in practice.}
There is no analog of this in the framework of Blom et al. 
It would be useful to get a deeper
understanding of the connection between the two approaches.

We conclude with a brief comparison of causal models with constraints
to the GSEMs (generalized structural equations models) of Peters and
Halpern \cite{PH21}.  
GSEMs are more expressive than causal models
with constraints (at least, if all variables have finite range); they
can simply express the effect of an 
intervention in a given context directly, by having a function
$\vec{F}$ that takes as input a context $\vec{u}$ and an intervention $I$, and
returns a set of states (intuitively, the set of states that might
result by performing intervention $I$ in context $\vec{u}$).
Thus, given a causal model with constraints $M$, we can define a GSEM
$M'$ that agrees with $M$ on all formulas in $\LS$ (which suffices,
given that we can replace all occurrences of the $\disc$ 
operator using the DSC axiom if all variables have finite range).
However, causal models with constraints allow us to describe constraints directly, 
which makes them more practical for many applications. 


\section*{Acknowledgments}

Halpern was supported in part by NSF grants IIS-178108 and IIS-1703846 and MURI grant
W911NF-19-1-0217 and ARO grant W911NF-22-1-0061. Sander Beckers was supported by the German Research Foundation (DFG)
under Germany’s Excellence Strategy – EXC number 2064/1 – Project number 390727645, and by the Alexander von Humboldt Foundation.

\bibliographystyle{acm}
\bibliography{joe}

\end{document}